\documentclass{amsart}

\usepackage{color}
\usepackage{hyperref}

\usepackage{tikz, tikz-cd}
\usetikzlibrary{arrows}
\usepackage{enumitem}

\usepackage{amssymb,amsfonts,mathrsfs,amsxtra,dsfont}
\usepackage{amsmath,amsthm,amscd}

\newcommand{\mor}[3]{\mathrm{Hom}_{#1}\left(#2,#3\right)}
\newcommand{\catname}[1]{\mathbf{#1}}

\newcommand{\sets}{\catname{Set}}

\newcommand{\covs}{\catname{Cov}}

\newcommand{\flag}{\begin{tikzpicture}
	\draw [thick] (0,0) -- ++(0,1.2ex) -- ++(.6ex,-.4ex) -- ++(-.6ex,-.4ex);
\end{tikzpicture}}
\newcommand{\covsfl}{\catname{Cov}_{_{\flag}}}
\newcommand{\sieves}{\catname{Sieve}}

\newcommand{\wghts}{\catname{Weight}}
\newcommand{\mets}{\catname{Met}}
\newcommand{\ults}{\catname{Ult}}
\newcommand{\cuts}{\catname{Cut}}
\newcommand{\trees}{\catname{Tree}}

\newcommand{\ants}[1]{\mathbf{A}^{\!#1}}
\newcommand{\antip}[1]{\mathit{A}^{#1}}

\newcommand{\fname}[1]{\mathcal{#1}}
\newcommand{\sle}{{\fname{E}}}
\newcommand{\prj}{{\fname{P}}}
\newcommand{\rips}{{\fname{R}}}
\newcommand{\cech}{\check{\fname{C}}}

\newcommand{\ml}{\fname{ML}}
\newcommand{\slc}{\fname{SL}}

\newcommand{\defn}[1]{\textbf{#1}}

\newcommand{\fat}[1]{\mathds{#1}}

\newcommand{\RR}{\fat{R}}
\newcommand{\RRplus}{\RR_{_{\geq 0}}}

\newcommand{\id}[1]{\mathrm{id}_{#1}}
\newcommand{\down}[1]{#1\!\downarrow}

\newcommand{\diam}[1]{\mathtt{diam}\!\left(#1\right)}
\newcommand{\dist}[2]{\mathtt{dist}\!\left(#1,#2\right)}

\newcommand{\injenv}[1]{\varepsilon(#1)} 

\newcommand{\THEN}{\;\Rightarrow\;}
\newcommand{\IFF}{\;\Leftrightarrow\;}

\newcommand{\inv}{^{{\scriptscriptstyle -1}}}
\newcommand{\set}[2]{\left\{#1\,\Big|\,#2\right\}}

\newcommand{\height}[1]{\left\Vert #1 \right\Vert}
\newcommand{\minset}[1]{\mathtt{Min}(#1)}

\usepackage{cleveref}

\theoremstyle{plain}
\newtheorem{theorem}{Theorem}
\newtheorem{proposition}[theorem]{Proposition}
\newtheorem{lemma}[theorem]{Lemma}
\newtheorem{corollary}[theorem]{Corollary}

\theoremstyle{definition}
\newtheorem{definition}[theorem]{Definition}

\theoremstyle{remark}
\newtheorem{example}[theorem]{Example}
\newtheorem{remark}[theorem]{Remark}

\begin{document}

\title{Functorial Hierarchical Clustering with Overlaps}
\author{Jared Culbertson}

\address{Sensors Directorate, Air Force Research Laboratory, 2241 Avionics Circle, Building 620
Wright--Patterson Air Force Base, Ohio 45433-7302, USA.}
\email{jared.culbertson@us.af.mil}

\author{Dan P. Guralnik}
\address{Electrical \& Systems Engineering Dept., University of Pennsylvania, 200 S. 33rd st., 203 Moore Building, Philadelphia, Pennsylvania 19104-6314, USA.}
\email[Corresponding author]{guraldan@seas.upenn.edu}

\author{Peter F. Stiller}
\address{Department of Mathematics, MS3368, Texas A\&M University, College Station, Texas 77843-3368, USA.}
\email{stiller@math.tamu.edu}

\maketitle

\begin{abstract}
This work draws inspiration from three important sources of research on dissimilarity-based clustering and intertwines those three threads into a consistent principled functorial theory of clustering.
Those three are the overlapping clustering of Jardine and Sibson, the functorial approach of Carlsson and M\'{e}moli to partition-based clustering, and the Isbell/Dress school's study of injective envelopes.
Carlsson and M\'{e}moli introduce the idea of viewing clustering methods as functors from a category of metric spaces to a category of clusters, with functoriality subsuming many desirable properties. 
Our first series of results extends their theory of functorial clustering schemes to methods that allow overlapping clusters in the spirit of Jardine and Sibson. This obviates some of the unpleasant effects of chaining that occur, for example with single-linkage clustering. We prove an equivalence between these general overlapping clustering functors and projections of weight spaces to what we term clustering domains, by focusing on the order structure determined by the morphisms. As a specific application of this machinery, we are able to prove that there are no functorial projections to cut metrics, or even to tree metrics. Finally, although we focus less on the construction of clustering methods (clustering domains) derived from injective envelopes, we lay out some preliminary results, that hopefully will give a feel for how the third leg of the stool comes into play.
\end{abstract}

\bigskip


\section{Introduction}

Problems surrounding data clustering have been studied extensively over the last forty years. 
Clustering stands as an important tool for analyzing and revealing the often hidden structure in data (and in today's big data) coming from fields as diverse as biology, psychology, machine learning, sociology, image understanding, and chemistry.
Among the earliest systematic treatments of clustering theory was that of Jardine and Sibson in 1971 \cite{js-1971}.
They laid out important desiderata for overlapping clustering methods and provided a relatively efficient algorithm for their so-called $B_k$ clustering which allowed overlapping clusters with no more than $k-1$ points in any overlap.
Since then, there have been several distinct directions of research in clustering theory, with only modest linkage between the methods of researchers pursuing different paths.

The classical work of Jardine and Sibson was followed by other similarly comprehensive works such as Everitt~\cite{everitt-2011}.
Further theoretical work on these mostly classical methods was also done by Kleinberg~\cite{kleinberg} and Carlsson and M\'{e}moli~\cite{cm-2010, cm-2013}.
Kleinberg in particular showed the incompatibility of a relatively simple set of desirable axioms for any partition based clustering method.
Carlsson and M\'{e}moli in turn introduced categorical language into partition based clustering and showed that single-linkage was (up to a scaling) the only functorial method satisfying all their axioms (which included the notions of representability and excisiveness) in the category of finite metric spaces with non-expansive maps.  

In another direction, work on computing phylogenetic trees inspired a seminal paper by Bandelt and Dress~\cite{bandelt-dress} on split decompositions of metrics.
This line of research was continued with investigations into split systems and cut points of injective envelopes of metric spaces.
Representative papers include \cite{dhm-2001} and \cite{dmsw-2013}.
While not explicitly clustering methods, these methods are quite similar in spirit to stratified/hierarchical clustering schemes.
In this genre, we might also add the classification of the injective envelopes of six-point metric spaces by Sturmfels and Yu \cite{sturmfels-yu}.
Bandelt and Dress have also had a large influence on another field as a result of their work on weak hierarchies~\cite{bd-1989,bd-1994}.
This led to work by Diatta, Bertrand, Barth\'{e}lemy, Brucker, and others on indexed set systems (see, e.g., \cite{bertrand-2000}, \cite{diatta-2007}, \cite{bd-2014}).
Another interesting development in this area is the work by Janowitz on ordinal clustering \cite{janowitz}.

Recently, with the emergence of the new field of topological data analysis (TDA), work has been done on topologically-based clustering methods.
This includes the Mapper algorithm by Singh, M\'{e}moli, and Carlsson~\cite{smc-2007}, as well as work on persistence-based methods \cite{cosg-2013} and Reeb graphs \cite{hrpbw-2012}.

Meanwhile, most users of clustering methods default either to a classical linkage-based clustering method (such as single-linkage or complete linkage) or to more geometrically-based methods like $k$-means.
Unfortunately, the wide array of clustering theories has had little impact on the actual practice of clustering.
Simply put, the gap between theory and efficient practice has been hard to bridge.

In this paper we draw inspiration from three of the sources mentioned above, and strive to intertwine those three threads into a consistent principled functorial theory of dissimilarity-based clustering.
Those three are the overlapping clustering of Jardine and Sibson, the functorial approach of Carlsson and M\'{e}moli, and the Dress school approach to clustering, via injective envelopes, which were independently discovered by Isbell and Dress.
This paper intends to fuse these approaches.
Our starting point is the paper of Carlsson and M\'{e}moli~\cite{cm-2013}  which introduces the idea of viewing clustering methods as functors
from a category of metric spaces to a category of clusters.
Many desirable properties of a clustering method are subsumed in functoriality when the morphisms are properly chosen.
Here the relevant morphisms under which the particular method is functorial can be viewed as giving restrictions on the allowable data processing operations - restrictions that impose consistency constraints across related data sets.
One of our first goals is to extend their theory of functorial clustering schemes to methods that allow overlapping clusters in the spirit of Jardine and Sibson, and in so doing obviate some of the unpleasant effects of chaining occurring in some linkage-based methods.  (See \cite{cm-2010}, Remark $16$, where Carlsson and M\'{e}moli discuss the chaining effects in single-linkage clustering and propose an alternate solution based on explicitly considering density.)
Rather than relying on chaining to overcome certain technical problems, we accept overlapping clusters.
This leads to a much richer set of possible clustering algorithms.   

Finally, although in this paper we focus less on the construction of clustering methods (clustering domains) derived from injective envelopes, we do in the final section lay out some preliminaries, that hopefully will enable the reader to get a feel for how the geometry of injective envelopes comes into play. In addition, lest the reader think we are all theory and no practice, we mention our ongoing algorithmic work on efficient implementations of some of these clustering schemes, along the lines of what has already been done for $q$-metrics by Segarra {\it et al.}~\cite{Segarra_et_al-q_metric_projections} and for dithered maximal linkage clustering by Gama {\it et al.}~\cite{Gama_Segarra_Ribeiro-dithering, gsr}.

\subsection{Weight Categories}
\begin{definition} Let $\wghts$ be the category of \defn{finite sets with weights}, whose objects have the form $(X,u)$ with $X$ a finite non-empty set and $u$ a symmetric non-negative map $u\colon X\times X\to\RR$, $(x,y)\mapsto u_{xy}$ satisfying $u_{xx}=0$ for all $x\in X$. A morphism $f\colon (X,u)\to(Y,v)$ is a set map $f\colon X\to Y$ such that $v_{f(x)f(x')}\leq u_{xx'}$; these will be referred to as \it{non-expansive maps}.
\end{definition}
For a fixed finite set $X$, we can define a local order structure by pointwise dominance on the full subcategory $\wghts_X$ of $\wghts$ consisting of weights on $X$.
In order to simplify notation, when the underlying set $X$ is fixed we will often refer to the object $(X, u) \in \wghts_X$ only by the weight function $u$ and state that $u \in \wghts_X$.
The set of objects of $\wghts_X$ is a partially ordered set (poset) with the ordering given by
\begin{displaymath}
	u\leq v \text{ if } u_{xy}\leq v_{xy} \text{ for all } x,y \in X.
\end{displaymath}

It will be convenient to denote, for any subset $U$ of the objects of $\wghts_X$,
\begin{displaymath}
	\down{U}:=\set{w\in\wghts_X}{\exists {u\in U}\;w\leq u}
\end{displaymath}
and in the case of the singleton set $\{u\}$, we will often write $\down{u}$ for $\down{\{u\}}$. For any subcategory $\catname{C}$ of $\wghts$, we will use the analogous notation $\catname{C}_X$ for $\catname{C} \cap \wghts_X$, which on objects will be the intersection of the objects of $\catname{C}$ and $\wghts_X$ with the morphisms of $\catname{C}$. 
Also, for every map of finite sets $f\colon X\to Y$ and $w\in\wghts_Y$ we define $f^\ast(w)\in\wghts_X$ to be the pullback of the weight $w$ to the set $X$, more explicitly, $f^\ast(w)_{xy}:=w_{f(x)f(y)}$. This notation allows another perspective on morphisms in $\wghts$: the map of finite sets $f\colon X\to Y$ induces a morphism $(X,u)\to(Y,v)$ in $\wghts$ if and only if $f^\ast(v)\leq u$. 

\begin{definition} We define a \defn{weight category} to be any full subcategory $\catname{C}$ of $\wghts$ such that $f^\ast\catname{C}_Y\subseteq\catname{C}_X$ for any map of finite sets $f\colon X\to Y$. We say that $\catname{C}$ is closed under pullbacks, and will sometimes refer to this as the pullback property. Note that this condition is only on the objects of the category $\catname{C}$. 
\end{definition}

Note that the category $\mets$ of finite metric spaces and non-expansive maps, first considered by Isbell~\cite{isbell} for arbitrary metric spaces, is a weight category in this sense. In this exposition, by metric we will always mean semimetric; {\em i.e.,} we will not require that distinct points have nonzero distance. 

\begin{remark}\label{rem:permutation invariance of weight categories} Since the pullback property allows arbitrary maps of finite sets, any permutation mapping $f\colon X\to X$ satisfies $f^\ast\catname{C}_X=\catname{C}_X$. In other words, weight categories are invariant under permutations. In this way, the pullback property is a generalization of what Jardine and Sibson term ``label freedom'' in~\cite{js-1971}, pp. 83--84.
\end{remark}

\begin{definition}\label{def:fibered} A functor $\fname{G}\colon\catname{C}\to\catname{D}$ between weight categories will be said to be \defn{fibered}, if the diagram

\begin{displaymath}
		\begin{tikzpicture}
			\node (c) at (0,0) {$\catname{C}$};
			\node (d) at (3,0) {$\catname{D}$};
			\node (set) at (3, -2) {$\sets$};
			
			\draw[->, above] (c) to node {$\fname{G}$} (d);
			\draw[->, below left] (c) to node {${\fname F}_{\catname{C}}$} (set);
			\draw[->, right] (d) to node {$\fname{F}_{\catname{D}}$} (set);
		\end{tikzpicture}
\end{displaymath}

commutes, where $\fname{F}_{\catname{C}}$ and $\fname{F}_{\catname{D}}$ are the respective forgetful functors. For such a functor $\fname G$ and any finite set $X$, we have an associated functor $\fname{G}_X\colon {\catname{C}}_X \to {\catname{D}}_X$. In other words, on objects, fibered functors fix the underlying set and on morphisms, fix the underlying set map. Intuitively, a fibered functor acts on weighted spaces only by deforming the weight structure. 
\end{definition}

\begin{remark}\label{rem:fibered}
This entire exposition could be reformulated in the setting of fibered categories and functors, where the fiber over a fixed finite set has the structure of a directed complete poset, but we did not view the additional advantage in abstraction and simplicity of definitions worth the overhead cost of framing the work in those terms. The fibration of a weight category over $\sets$ is straightforward. In \Cref{sec:a_spaces}, we give a hint at where this generality would be important: we explore a subcategory of $\wghts$ with a restricted set of morphisms. The important thing in this type of setting is that the cartesian morphisms in each fiber provide this directed complete poset structure. For $\wghts$, the categories considered in \Cref{sec:a_spaces}, as well as for the categories considered by Carlsson and M\'{e}moli~\cite{cm-2013}, this amounts to noting that the morphisms with underlying map the identity are cartesian.
\end{remark}

\section{Clustering with Overlaps: Sieves}\label{section:sieves}

Let $\covs$ be the category of coverings, with an object a pair $(X,C)$ consisting of a finite set $X$ together with a cover $C$ of $X$. A morphism $f\colon (X,C) \to (Y,D)$ is a set map $f\colon X \to Y$ such that the cover $f^{-1}(D)$ of $X$ is refined by $C$.  

\begin{definition}\label{def:flag cover}
Given a non-empty finite set $X$ we define a \defn{non-nested flag cover} of $X$ to be a cover $C$ of $X$ additionally satisfying:
\begin{enumerate}
	\item[(i)] for all $A,B \in C$ with $A \subseteq B$, we have $A = B$;
	\item[(ii)] the abstract simplicial complex $K_C$ with vertices corresponding to elements of $X$ and faces all subsets of the sets in $C$ is a flag complex, {\em i.e.,} if all of the edges of a given simplex are in $K_C$, then that simplex itself is also in $K_C$. 
\end{enumerate} 

The category $\covsfl$ is then the full subcategory of $\covs$ where the covers are required to be non-nested flag covers.
\end{definition}

\begin{remark}\label{rem:local subcat of covf}
For a non-nested flag cover $C$ of $X$, observe that the sets in $C$ provide the maximal simplices in the simplicial complex defined in $(ii)$ above. Given a fixed finite set $X$, we will denote the full subcategory of $\covsfl$ consisting of covers on $X$ by $\covsfl(X)$.  Notice that, in particular, every partition of $X$ is a flag cover of $X$.
\end{remark}
\begin{remark}\label{rem:flagification}
The inclusion functor $\covsfl \to \covs$, which simply forgets that a given cover is a non-nested flag cover,  has a left adjoint given by flagification. Formally, any cover $\tilde{C}$ of a finite set $X$ has a unique associated non-nested flag cover: $\tilde{C}$ can be ``flagified'' to a non-nested flag cover $C$ in a minimal way. Explicitly, $\tilde{C}$ will refine the non-nested flag cover $C$ and $C$ will refine any other non-nested flag cover which $\tilde{C}$ refines. Algorithmically, this is very simple and amounts to adding in any subsets required by the flag condition ({\em i.e.}, adding a subset $A$ if all of the proper subsets of $A$ are already in $\tilde{C}$, then iterating this process) and then removing any subsets from the cover that are entirely contained in another subset. In principle, this construction would seemingly be important for defining the morphisms in $\covsfl$. That is, from a categorical point of view, the refinement criterion for a map of finite sets $f\colon X \to Y$  to be a $\covsfl$-morphism $f\colon (X, C) \to (Y,D)$ is not natural in the sense that the cover $f^{-1}(D)$ of $X$ might contain nesting and so $(X, f^{-1}(D))$ is not an object in $\covsfl$. The more appropriate condition involves first applying the flagification functor to $(X, f^{-1}(D))$ and then verifying the refinement condition. However, it is straightforward to see that these are equivalent: if $C$ is a flag cover on $X$, then $f^{-1}(D)$ is refined by $C$ if and only if the flagification of $f^{-1}(D)$ is refined by $C$. 
\end{remark}

\begin{definition}\label{def:sieve} A \defn{sieve} on a finite set $X$ is a function $\theta\colon \RR_{\geq 0} \to \covsfl(X)$ such that:
\begin{enumerate}
	\item[(i)] If $t_1 < t_2$ then $\theta (t_1)$ refines $\theta (t_2)$;
	\item[(ii)] For any $t$, there is an $\varepsilon>0$ such that $\theta (t^\prime )=\theta_X (t)$ for all $t^\prime \in[t,t+\varepsilon )$;
	\item[(iii)] There exists $t \in \RR_{\geq 0}$ such that $\theta (t)$ is the trivial cover $\{X\}$.
\end{enumerate}
We say that a sieve $\theta$ on $S$ is \defn{proper}, if $\theta(0)$ is the covering by singletons. In order to curb the proliferation of parentheses, we will also use the notation $\theta_t$ for $\theta(t)$. 
\end{definition}

\noindent Sieves are the obvious generalization of {\em dendrograms} (see \cite{cm-2010}, Section $4.1$ for a modern, self-contained discussion of dendrograms), which satisfy the same conditions as sieves, but take values in the more restricted set of partitions of $X$ rather than coverings of $X$. Jardine and Sibson~(\cite{js-1971}, p. $80$) also considered a similar construction to our sieves, though taking values in symmetric reflexive relations on $X$. They also did not focus on the morphisms between sieves as we do, starting with the following definition.

\begin{definition}\label{def:sieve category}
We define the \defn{category of sieves}, $\sieves$, as the category of pairs $(X,\theta)$, where $\theta \colon \RRplus \to \covsfl(X)$ is a sieve on a finite set $X$. 
The morphisms in $\sieves$ are an extension of the morphisms of $\covsfl$; that is, a map of finite sets $f\colon X \to Y$ is a morphism of sieves $(X,\theta) \to (Y,\psi)$ if for every $t \in \RR_{\geq 0}$, $\theta_t$ refines $f^{-1}(\psi_t)$. In other words, for each $t \in \RRplus$, we have a functor $\sieves \to \covsfl$. Note that just as in \Cref{rem:flagification}, in general, $(X, f^{-1}(\psi_t))$ might be an object of the ambient category $\covs$ and not $\covsfl$ due to nesting, but nonetheless, this condition is equivalent to a more complex condition involving first flagifying $f^{-1}(\psi_t)$. 
\end{definition}

Just as in $\wghts$, the categories $\covsfl$ and $\sieves$ are also fibered over $\sets$ in a natural way. We will restrict to fibered functors in our discussion, which in particular ensures that underlying sets and mappings are fixed by the functors. Any fibered functor with values in $\sieves$ then provides a hierarchical method of clustering that respects the constraints imposed by the morphisms in $\wghts$ and produces potentially overlapping clusters at any fixed scale. For clarity in the exposition, we make the following definition. 

\begin{definition}
A \defn{sieving functor} is a fibered functor $\wghts \to \sieves$. One of the primary purposes of this paper is to characterize a particularly ``nice'' class of sieving functors, in a way that will be made precise below (see \Cref{thm:correspondence}). 
\end{definition}

\begin{example}[Rips Sieving]\label{ex:Rips sieve}
Let $(X, u) \in \wghts$ and let $\delta\geq 0$. Recall that the Rips complex at resolution $\delta$ on $(X,u)$ is the abstract simplicial complex $K_u(\delta)$ with vertex set $X$, where a subset $A\subseteq S$ forms a face of $K_u(\delta)$ if and only if $\diam A\leq\delta$ (see~\cite{edelsbrunner-harer}, Section III.2 for a more detailed treatment of the Rips complex and other related ideas). If $M_u(\delta)$ is the collection of maximal simplices of $K_u(\delta)$, then $M_u(\delta)$ forms a non-nested flag cover of $X$. Define $\fname{R}\colon \wghts \to \sieves$ by 
\begin{displaymath}
    \fname{R}(X, u)_t = (X, M_u(t)),
\end{displaymath}
which we will call the Rips sieving functor (elsewhere, including in \cite{cghs-consistency}, this functor is denoted $\ml$ and called maximal-linkage). On morphisms, $\fname{R}$ sends a non-expansive map of weight spaces to the same underlying set function. It is straightforward to see that this gives a $\sieves$-morphism.
\end{example}

\begin{proposition}\label{prop:equivalence}
    The category $\wghts$ is equivalent to $\sieves$.
\end{proposition}
\begin{proof}
Given a sieve $(X, \theta)$, define a weight space $(X, u_\theta)$ pointwise with $x,y \in X$ by
    \begin{displaymath}
        u_\theta(x,y) = \min\set{t \in \RRplus}{\exists A \in \theta_t, \{x,y\} \subseteq A}.
    \end{displaymath}
At first glance, it might appear that we need to consider the infimum rather than the minimum, but condition $(ii)$ in \Cref{def:sieve} ensures that $u_\theta$ is well-defined. This assignment extends to a functor $\fname{J}\colon \sieves \to \wghts$, and it is straightforward to check that $\fname{J}\circ\fname{R}, \fname{R}\circ\fname{J}$ are the respective identity functors. 
\end{proof}

\begin{remark}\label{rem:rips}
This proposition shows that Rips sieving is the canonical way of producing sieves from weight spaces, in the sense that a weight space can be recovered from its Rips sieve. The result extends the connection noted by Carlsson--M\'{e}moli between the category of ultrametrics (\Cref{def:ultrametrics}) and the category of dendrograms; indeed our proof is a reworking of theirs in this new context (see \cite{cm-2010}, Theorem $9$).  From this perspective, a natural question is how to characterize sieving functors that factor through Rips sieving, {\em i.e.}, sieving functors $\fname{C}$ which can be realized as a composition $\fname{C} = \rips \circ \fname{P}$, for some functor $\fname{P}$. \Cref{thm:correspondence} below shows one result in this direction. 
\end{remark}

\begin{example}[Single-Linkage Sieving]\label{ex:SL sieve}
Let $(X, u) \in \wghts$ and $K_u(\delta)$ as in \Cref{ex:Rips sieve}. In this case, rather than look at maximal simplices, let $C_u(\delta)$ be the connected components of $K_u(\delta)$. Then we can define the single-linkage sieving functor $\slc\colon \wghts \to \sieves$ by
\begin{displaymath}
	\slc(X,u)_t = (X, C_u(t)),
\end{displaymath}
with morphisms being sent to the same underlying set maps. This is equivalent to the many alternative well-known characterizations of single-linkage clustering, for instance in \cite{sibson-73,Gower-Ross:mst,cm-2013}. 
\end{example}

\begin{example}[\v{C}ech Sieving]\label{ex:cech sieve}
If $(X, u) \in \wghts$, define a graph $G_u(\delta)$ with vertices the elements of $X$ and an edge $xy$ if there exists an element $z \in X$ with $u(x,z), u(z,y) \leq \delta$. We then define $\cech\colon \wghts \to \sieves$ to be the assignment which sends any $t \in \RRplus$ to the set of maximal cliques of $G_u(t)$. Again in this case, morphisms are sent to the same underlying set maps.
\end{example}

Many other examples can be derived from the clustering functors introduced in our work with Hansen in \cite{cghs-consistency}. 

\section{Clustering Domains}\label{section:clustering domains}
\begin{definition}\label{def:clustering domain} We say that a weight category $\catname{D}$ is a \defn{clustering domain} if  for any non-empty finite set $X$ and any $w \in \wghts_X$, the set $\catname{D}_X \cap \down{w}$ is non-empty and $\sup$-closed in $\wghts_X$. That is, 
\begin{displaymath}
	\left(X, \sup \{ u \in {\catname{D}_X\cap\down{w}}\}\right) \in \catname{D},
\end{displaymath}
for all $w \in \wghts_X$. 
\end{definition}

The importance of sup-closure was recognized as early as~\cite{js-1971}, in their discussion of optimality (p. $85$), and it will often be convenient to satisfy a stronger (but more easily verified) condition related to the topology on $\wghts$ induced by the identification of the objects of $\wghts_X$ with the nonnegative real orthant $\RR_{\geq 0}^{\binom{|X|}{2}}$.

\begin{proposition}\label{prop:closed_max_closed}
Let $\catname{C}$ be a weight category. If the set of objects of $\catname{C}_X$ is closed as a subset of $\RR_{\geq 0}^{\binom{|X|}{2}}$ and $\catname{C}_X$ is closed under taking finite maxima, then $\catname{C}$ is a clustering domain.  
\end{proposition}
\begin{proof}
If $w \in \wghts_X$ and $S\subset\catname{C}_X$ satisfies $d\leq w$ for all $d\in S$, then $v:=\sup_{d\in S}(d)$ has $v\leq w$ and we consider the following procedure. For every pair $x,y\in X$ and $\varepsilon>0$ find $u^\varepsilon_{x,y}\in S$ such that $v(x,y)-u^\varepsilon_{x,y}(x,y)<\varepsilon$; then $u^\varepsilon:=\max_{x,y\in X}u^\varepsilon_{x,y}\in\catname{C}_X$ and satisfies $\left\Vert v-u^\varepsilon\right\Vert_\infty<\varepsilon$, and so we conclude that $v$ is in the closure of $\catname{C}_X$.
\end{proof}

\subsection{Projections}\label{section:projections}
The reason to define clustering domains is that, in some sense, clustering maps ``live'' on them.
\begin{definition}[Canonical Projection]\label{def:can_projection} Given a clustering domain $\catname{D}$, we define the \defn{canonical projection} $\prj_\catname{D}\colon \wghts\to\wghts$ to be the fibered functor defined by
\begin{displaymath}
	\prj_\catname{D}(X,w)=\left(X,\sup\{u\in\catname{D}_X\cap\down{w}\}\right).
\end{displaymath}
Recall that the fibered condition implies that $\prj_\catname{D}$ sends a morphism in $\wghts$ to the morphism with the same underlying set map. Also $\left(X,\sup\{u\in\catname{D}_X\cap\down{w}\}\right)$ is an object in $\catname{D}$ since $\catname{D}$ is a clustering domain. 
\end{definition}
We need to verify the assertions made in this definition regarding properties of $\prj_\catname{D}$.
\begin{proposition}[Properties of the Canonical Projection] Suppose $\catname{D}$ is a clustering domain and $\prj=\prj_\catname{D}$ is the associated canonical projection. Then $\prj$ is a fibered endofunctor of $\wghts$ satisfying the additional properties:
\begin{enumerate}
	\item $\prj\circ \prj=\prj$,
	\item $\prj_Xw\leq w$ for all sets $X$ and $w\in\wghts_X$.
\end{enumerate}
\end{proposition}
\begin{proof} Since the identities
\begin{displaymath}
	 \prj(\id{(X,w)})=\id{\prj(X,w)}\,,\quad
	 \prj(g\circ f)=\prj(g)\circ \prj(f)
\end{displaymath}
are immediate from the definition, one is left only to verify that $\prj(f)$ is a non-expansive map whenever $f\colon (X,u)\to(Y,v)$ is. Equivalently, we need to show that $f^\ast(\prj v)\leq \prj u$ holds whenever $f^\ast(v)\leq u$. By definition, $\prj v\leq v$ and we have:
\begin{displaymath}
	\prj v\leq v
		\IFF f^\ast(\prj v)\leq f^\ast(v)
		\THEN f^\ast(\prj v)\leq u
		\IFF f^\ast(\prj v)\in\down{u}
\end{displaymath}
Thus, $f^\ast(\prj v)\in\down{u}\cap f^\ast(\catname{D}_Y)\subseteq\down{u}\cap \catname{D}_X$, by the definition of a clustering domain. Finally, $f^\ast(\prj v)\leq \prj u$ by the definition of $\prj u$.
\end{proof}
In an attempt to construct more general sieving functors in the spirit of the previous proposition (and the similar conditions of Jardine and Sibson in~\cite{js-1971}, pp. 83--85) one might seek to introduce the following definition:
\begin{definition}[Projection]\label{def:projection} Let $\catname{C}$ be a weight category. We say that a fibered functor $\prj\colon \catname{C}\to\catname{C}$ is a \defn{projection}, if it satisfies:
\begin{itemize}
	\item{\bf Idempotency} $\prj\circ \prj=\prj$;
	\item{\bf Contraction} $\prj_Xw\leq w$ for all sets $X$ and $w\in\catname{C}_X$.
\end{itemize}
\end{definition}
\begin{remark}\label{functorial is order-preserving} The fact that $\prj$ is a fibered functor implies that for a fixed set $X$, the self-maps $\prj_X\colon \wghts_X\to\wghts_X$ are order-preserving. For all $w_1,w_2\in\wghts_X$ one has:
\begin{eqnarray*}
	w_1\leq w_2	&\IFF& \id{X}\in\mor{\wghts}{(X,w_2)}{(X,w_1)}\\
		&\THEN& \id{X}=\prj(\id{X})\in\mor{\wghts}{(X,\prj_Xw_2)}{(X,\prj_Xw_1)}\\
		&\IFF&	\prj_Xw_1\leq \prj_Xw_2
\end{eqnarray*}
\end{remark}

The following result was initially observed for single linkage hierarchical clustering in~\cite{Zhu_et_al-SLHC_properties}, however its categorification and generalization are new:
\begin{proposition}[Uniqueness of Projections]\label{prop:unique_projections} Let $\catname{C} \subseteq \wghts$ be a weight category. Then every weight category $\catname{D}\subseteq\catname{C}$ admits at most one projection $\prj\colon \catname{C}\to\catname{C}$ whose image coincides with $\catname{D}$.
\end{proposition}

\begin{proof} Suppose $\prj,\fname{Q}$ are arbitrary projections with image $\catname{D}$. Fixing a non-empty finite set $X$ and applying the idempotency requirement, we have that for all $w\in\catname{C}_X$:
\begin{displaymath}
	w\in\catname{D}_X\IFF \prj_Xw=w\IFF \fname{Q}_Xw=w.
\end{displaymath}
We conclude that $\prj_X$ fixes $\fname{Q}_Xw$ and $\fname{Q}_X$ fixes $\prj_Xw$ for all $w\in\catname{C}_X$. From \Cref{functorial is order-preserving} we then obtain, for all $w\in\catname{C}_X$:
\begin{displaymath}
	\prj_Xw\leq w\THEN \prj_Xw=\fname{Q}_X\prj_Xw\leq \fname{Q}_Xw
\end{displaymath}
We conclude that $\prj_Xw\leq \fname{Q}_Xw$ for all $w\in\catname{C}_X$. By a symmetric argument, the reverse inequality holds true as well and we have shown that $\prj$ and $\fname{Q}$ coincide on $\catname{C}$.
\end{proof}

\begin{corollary} Suppose $\catname{C}$ is a weight category containing a clustering domain $\catname{D}$. Then $\catname{C}$ admits one and only one projection with image $\catname{D}$: the restriction of the canonical projection $\prj_{\catname{D}}$ to $\catname{C}$.\hfill\qedhere
\end{corollary}
This last corollary emphasizes that the existence of a clustering projection must be characterized in terms of its image category. We have:

\begin{theorem}[Existence of Projections]\label{prop:projections_existence} Let $\catname{D}$ be a full subcategory of the category $\wghts$. Then $\catname{D}$ is a clustering domain if and only if it is the image of a projection functor $\prj\colon \wghts\to\wghts$.
\end{theorem}
\begin{proof} If $\catname{D}$ is a clustering domain then it is the image of the canonical projection $\prj=\prj_\catname{D}$.
Conversely, assume the full subcategory $\catname{D}$ is the image of a clustering projection functor $\prj\colon \wghts\to\wghts$ and let us prove it is a clustering domain. 

First, we need to see that $\catname{D}$ is a weight category. Suppose $f\colon X\to Y$ is a map of finite sets. For $w\in\catname{D}_Y$, consider $v=f^\ast w\in\wghts_X$: then $f\colon (X,v)\to(Y,w)$ is a morphism and hence $f\colon (X,\prj_Xv)\to(Y,\prj_Yw)=(Y,w)$ is a morphism as well; or equivalently, $v=f^\ast w\leq \prj_Xv$. Thus, by the contraction property of $\prj$ we have $v=\prj_Xv\in\catname{D}_X$, as desired.

Note that for any finite set $X$ and $w \in \wghts_X$, the set $\catname{D}_X \cap \down{w}$ is nonempty since in particular, this contains $\prj w$. It remains to prove that $\catname{D}_X$ is $\sup$-closed. Let $U$ be the set of objects in $\catname{D}_X \cap \down{w}$.  If $v = \sup\{u \in U\}$  then $v$ is the minimal weight in $\wghts_X$ satisfying $v \geq u$ for all $u \in U$.  By functoriality, we see that $\prj v \geq \prj u = u$, so we must have $\prj v \geq v$. Since $\prj$ is a projection, and in particular, satisfies the contraction property, $\prj v \leq v$ as well and so $v \in \catname{D}_X$, as required. 
\end{proof}

\begin{remark}
Let $\catname{C} \subset \catname{D}$ be clustering domains. \Cref{prop:unique_projections} implies that we have a unique factorization:
\begin{displaymath}
	\begin{tikzpicture}
		\node (w) at (0,4) {$\wghts$};
		\node (d) at (0,2) {$\catname{D}$};
		\node (c) at (3,2) {$\catname{C}$};
		
		\draw[->, left] (w) to node {$\prj_\catname{D}$} (d);
		\draw[->, below] (d) to node {$\mathrm{Res}_{\catname{D}}({\prj_\catname{C}})$} (c);
		\draw[->, above right] (w) to node {$\prj_\catname{C}$} (c);
	\end{tikzpicture}
\end{displaymath}
where $\prj_C$ and $\prj_D$ are the respective projections associated with $\catname{C}$ and $\catname{D}$ while $\mathrm{Res}_{\catname{D}}({\prj_\catname{C}})$ is the restriction of the functor $\prj_C$ to the subcategory $\catname{D}$. Note that the essential portions of the proof of the proposition depend only on the order structure of $\wghts_X$, and consequently, this factorization holds even when replacing $\sets$ with a category with fewer morphisms (see \Cref{rem:fibered}), such as only injective maps or only surjective maps. 
\end{remark}

\begin{definition}\label{def:stationary}
Let $\fname{C}\colon \wghts \to \sieves$ be a sieving functor and recall the functor $\fname{J}\colon \sieves \to \wghts$ from the proof of \Cref{prop:equivalence}. Then $\fname{C}$ will be called \defn{stationary} if $\fname{J} \circ \fname{C}$ is a projection.
\end{definition} 

Note that the stationary functors that we consider in this paper do not in general satisfy the straightforward generalizations of either representability or excisiveness (see~\cite{cm-2013}) to the setting of coverings. However, in Theorem $6.3$ of that paper, Carlsson and M\'{e}moli show that their notion of finite representability (and therefore also excisiveness) implies a factorization into an endofunctor followed by single-linkage. Thus, our notion of stationarity is analogous to these conditions in the sense of the following theorem.   

\begin{theorem}\label{thm:correspondence}
There is a bijective correspondence between the collection of stationary sieving functors and the collection of clustering domains: every stationary sieving functor factors uniquely through a clustering projection and the Rips sieving functor $\rips$ restricted to a clustering domain. Pictorially, for every stationary sieving functor $\fname{C}$ there is a commutative diagram 
\begin{displaymath}
	\begin{tikzpicture}
		\node (w) at (0,0) {$\wghts$};
		\node (s) at (3, -2) {$\sieves$.};
		\node (d) at (0,-2) {$\catname{D}_{\fname{C}}$};
		
		\draw[->, above right] (w) to node {$\fname{C}$} (s);
		\draw[->, left] (w) to node {$\prj_{\catname{D}_{\fname{C}}}$} (d);
		\draw[->, below] (d) to node {$\mathrm{Res}_{\catname{D}_{\fname{C}}}(\rips)$} (s);
	\end{tikzpicture}
\end{displaymath}
where $\prj_{\catname{D}_{\fname{C}}}$ is the unique associated clustering projection and $\mathrm{Res}_{\catname{D}_{\fname{C}}}(\rips)$ is the Rips sieving functor restricted to $\catname{D}_{\fname{C}}$. 
\end{theorem}
\begin{proof}
If $\fname{C}\colon \wghts \to \sieves$ is a stationary sieving functor, then $\prj_{\catname{D}_{\fname{C}}} = \fname{J} \circ \fname{C}$ is a projection and so from \Cref{prop:projections_existence} we see that the image $\catname{D}_{\fname{C}}$ of $\prj_{\catname{D}_{\fname{C}}}$ is a clustering domain. The theorem then follows from \Cref{prop:equivalence}. 
\end{proof}

In our first look at sieving functors, we defined the \v{C}ech sieve of \Cref{ex:cech sieve}, which is not a stationary sieving functor: for a fixed weight space, after sufficiently many applications (interleaving applications of the functor $\fname{J}$) it gives the result of the single-linkage sieve (see the discussion of ${\mathbf{L}^k}$ clustering in our previous paper with Hansen~\cite{cghs-consistency}, where we describe \v{C}ech sieving in terms of set relations). In that paper, we exhibited many different sieving functors that fail to be stationary. An interesting generalization of the work here would be to explore more complex characterizations that would include these non-stationary sieving functors. 

\subsection{Single-linkage as a projection}\label{sec:single_linkage}
In this section, we illustrate the relationship between projections and clustering domains by revisiting the well-known single-linkage clustering method from a new perspective. More details than necessary are given in order to provide a template for one way of using the results in the previous section to construct a clustering projection. 

Let $(X, u) \in \wghts$ and $x_1 ,x_2 \in X$ with $x_1 \neq x_2$. Also, let $\Delta^{\varepsilon}_{2}$ be the two point space $\{1,2\}$ with distance $\varepsilon$ between the two points. Then define $\sle_u (x_1 ,x_2 )$ to be the largest $\varepsilon$ for which there exists a morphism (non-expansive map) $f\colon (X,u )\to\Delta^{\varepsilon}_{2}$, with $f(x_i )=i\text{ for }i=1,2$.  For $x_1 = x_2$, define $\sle_u(x_1, x_2) = 0$. Note that if the smallest distance between any two points in $X$ is $u_{\text{sep}}$ (the so-called separation 
of $X$), then $\sle_u (x_1 ,x_2 )\geq u_{\text{sep}}$ since we can map $X$ to $\Delta^{u_{\text{sep}}}_{2}$ arbitrarily and have a non-expansive map. Moreover, notice that $\sle_u(x_1, x_2) \leq u(x_1,x_2)$ since $f$ is required to be non-expansive. This means that the identity set map on $X$ gives a $\wghts$-morphism $(X,u )\to(X,\sle_u)$, or in the language of \Cref{def:projection}, $\sle$ is a contraction. 

\begin{definition}\label{def:cut_metric}
Recall that a \defn{cut metric} (see \cite{Deza-Laurent:cuts_metrics} for a detailed treatment) on any non-empty set $Y$ is a non-negatively weighted finite combination of cuts.  A cut is a metric given by a subset 
$A\subset X$ and its complement $A^c$, with the associated metric $\delta_{A}$ on $X$ is defined as
\[
 	\delta_{A}(y_1 ,y_2 ) =
 	\begin{cases}
  		0 & \text{if }y_1 ,y_2 \text{ are both in }A\text{ or both in }A^c\\
  		1 & \text{otherwise.}
 	\end{cases}
\]
Notice that $\delta_{A}=\delta_{A^c}$ and the trivial cut $\{X,\emptyset\}$ produces the zero metric.
\end{definition}

Now for any fixed $x_1, x_2 \in X$, and any morphism $f\colon (X,u )\to\Delta^{\sle_u (x_1 ,x_2 )}_{2}$ satisfying $f(x_i) = i$ for $i = 1, 2$, the morphism $f$ defines a bi-partition of $X$ into $X_1 =f^{-1}(1)$ and $X_2 =f^{-1}(2)$ with $x_1 \in X_1$ and $x_2 \in X_2$.  The morphism $f$ will then factor through $X$ with the cut metric $\delta_{\{X_1\}} = \delta_{\{X_2\}}$ scaled by 
$\sle_u (x_1 ,x_2 )$:
\[
\begin{tikzcd}
 	(X,u ) \arrow{rd} \arrow{r}{f} & \Delta^{\sle_u (x_1 ,x_2 )}_{2}\\
 	& (X,\sle_u (x_1 ,x_2 ) \delta_{\{X_1\}})\,. \arrow{u}
\end{tikzcd}
\]
Notice that $\sle_u(x_1,x_2)\delta_{\{X_1\}}$ is just the pullback of the metric $\Delta^{\sle_u(x_1,x_2)}_2$ by $f$. In general, we might have several morphisms $f\colon (X,u)\to\Delta^{\sle_u (x_1 ,x_2 )}_{2}$ which give different bi-partitions.  

We now show that the mapping $\sle$ taking $(X, u)$ to $(X, \sle_u)$ is a projection as described in \Cref{def:projection}.

\begin{proposition}
The mapping $\sle\colon \wghts \to \wghts$ is a fibered functor with the contraction property, where $\sle$ takes a morphism $f\colon (X,u) \to (Y, v)$ in $\wghts$ to the same underlying set map.
\end{proposition}

\begin{proof}
 Since $\sle$ takes morphisms to the same underlying set map, it is fibered, and as has already been observed above, $\sle$ is a contraction. Hence the only thing requiring proof is that the map $\sle(g)\colon(X,\sle_u) \to (Y, \sle_v)$ is a $\wghts$-morphism. We simply need to see that for every pair of points $x_1 ,x_2 \in X$ we have $\sle_u(x_1 ,x_2 ) \geq \sle_v (g(x_1 ),g(x_2 ))$.  But we have morphisms $g$ and $f$:
\begin{displaymath}
 	(X,u )\xrightarrow{g}(Y,v)\xrightarrow{f}\Delta^{\sle_v (f(x_1 ),f(x_2 ))}_{2}
\end{displaymath}
whose composition $f\circ g$ is a morphism sending $x_1$ and $x_2$ to different points in the set $\Delta^{\sle_v (f(x_1 ),f(x_2 ))}_{2}$.  Since $\sle_u (x_1 ,x_2 )$ is defined to be the maximum $\varepsilon$ satisfying this property, we must have $\sle_u (x_1 ,x_2 ) \geq \sle_v (g(x_1 ),g(x_2 )).$ 
\end{proof}

\begin{proposition}\label{prop:sl_idempotent}
	The functor $\sle$ is idempotent, i.e. $\sle \circ \sle = \sle$.
\end{proposition}

\begin{proof}
 Since $\sle$ sends morphisms to the same underlying set map, we only need to check that $\sle$ is idempotent on objects. Notice that for any $\varepsilon > 0$, we have $\sle(\Delta^{\varepsilon}_2) = \Delta^{\varepsilon}_2$. Thus for any morphism $f\colon (X, u) \to \Delta^{\varepsilon}_2$ separating a pair of points $x_1, x_2 \in X$, we have a corresponding morphism $\sle(f)\colon (X, \sle_u) \to \Delta^{\varepsilon}_2$ also separating $x_1, x_2$. Combining this with the fact that $\sle^2(x_1, x_2) \leq \sle(x_1, x_2)$, we get equality. 
 \end{proof}

The previous two propositions then show that $\sle$ is a projection and so by \Cref{prop:unique_projections}, the functor $\sle$ is the canonical projection for its image, which must be a clustering domain. The remaining question is how to characterize this clustering domain. The following lemma will be useful toward this end. 

\begin{lemma}\label{lem:two_points}
	Let $\catname{D}$ be a clustering domain containing $\Delta^{\varepsilon}_2$ for all $\varepsilon \geq 0$. Then $\mathrm{im}(\sle) \subseteq \catname{D}$.
\end{lemma}
\begin{proof} Let $\prj_{\catname{D}}\colon \wghts \to \wghts$ be the canonical projection with image $\catname{D}$ and $(X, u) \in \wghts$. Then for any $x_1, x_2 \in X$, we have 
\begin{displaymath}
	\sle_u(x_1 ,x_2 ) \leq \prj_{\catname{D}}(u)(x_1 ,x_2 ) \leq u(x_1 ,x_2 )
\end{displaymath}
since $f\colon(X,u)\to\Delta^{\sle_u(x_1 ,x_2 )}_{2}$ yields a morphism $\prj_{\catname{D}}(f)\colon(X,\prj_{\catname{D}}(u)) \to \Delta^{\sle_u(x_1 ,x_2 )}_{2}$.  \Cref{prop:sl_idempotent} then implies that $\sle \circ \prj_{\catname{D}} = \sle$. Replacing $u$ with $\sle_u$ in the above reasoning, we also get 
\begin{displaymath}
	\sle_u(x_1, x_2) \leq \prj_{\catname{D}}(\sle_u)(x_1, x_2) \leq \sle_u(x_1, x_2),
\end{displaymath}
and we see that $\prj_{\catname{D}}$ fixes $\mathrm{im}(\sle)$. 
\end{proof}

For further discussion of this point in a different context, see \cite{cghs-consistency}, Theorem $8$. Here \Cref{lem:two_points} extends the concept of the referenced theorem to the hierarchical setting, where the infrastructure of clustering domains and projections affords a succinct statement.  
In order to state the next result, we need to recall the definition of an ultrametric.

\begin{definition}\label{def:ultrametrics}
An \defn{ultrametric} on a set $X$ is a metric $(X,u)$ satisfying a stronger version of the triangle inequality: for any $x,y,z$ in $X$, we have
\begin{displaymath}
	u(x,z) \leq \max \{u(x,y), u(y,z)\}.
\end{displaymath}
The category of finite ultrametrics $\ults$ is the full subcategory of $\wghts$ with objects the finite ultrametric spaces. 
\end{definition}

\begin{corollary}
 For any $(X, u) \in \wghts$, the weight $(X,\sle_u)$ is an ultrametric.
\end{corollary}

\begin{proof}
Following \Cref{lem:two_points}, it is sufficient to notice that the category $\ults$ of ultrametrics contains all two point weight spaces. Alternatively, we could verify the ultrametric inequality directly. Let $x_1, x_2 \in X$ with $x_1 \neq x_2$, let $y\in X$ and let $f\colon (X,u) \to \Delta^{\sle_u(x_1, x_2)}_2$ be a $\wghts$-morphism that separates $x_1$ and $x_2$, so that $f(x_i)=i$. Now, if $y\in f\inv(1)$ then $\sle_u (x_2 ,y) \geq \sle_u (x_1,x_2)$ by construction. Likewise, if $y \in f\inv(2)$ then we have $\sle_u (x_1 ,y) \geq \sle_u (x_1,x_2)$. Since $y$ is in one set or the other, we must have $\sle_u (x_1 ,x_2 ) \leq \max\{\sle_u (x_1 ,y),\sle_u (x_2 ,y)\}$ as desired.
\end{proof}

\begin{proposition}\label{prop:sle_ult}
The image of $\sle$ is precisely the category $\ults$ of ultrametrics. 
\end{proposition}

\begin{proof} Following the previous corollary, we need to see that every ultrametric $(X,u)$ is in the image of $\sle$ (equivalently, that $u$ is fixed by $\sle$). Given an ultrametric $(X, u)$, and $x_1 ,x_2 \in X_1$ with $x_1 \neq x_2$, it suffices to exhibit a non-expansive map from $(X,u)$ to $\Delta^{u(x_1, x_2)}_{2}$.  If $u(x_1, x_2) = 0$ any map will do, provided $x_i$ is sent to $i$.   
 
The ultrametric property allows us to define an equivalence relation on $X$ by $y_1 \sim y_2$ if $u(y_1 ,y_2 ) < u(x_1, x_2)$.
 
We can construct our morphism $f\colon(X,u)\to\Delta^{u(x_1, x_2)}_{2}$ as follows.  Since $x_1$ and $x_2$ are in different equivalence classes we map everything equivalent to $x_1$ to $1$ and everything equivalent to $x_2$ to $2$.  The remaining equivalence classes can be 
 mapped arbitrarily as long as the entire equivalence class is sent to the same point in $\Delta^{u(x_1, x_2)}_{2}$.
 
Because points in different equivalence classes are at least $u(x_1, x_2)$ apart, we see that $f$ is non-expansive.
\end{proof}

Since we know that single-linkage clustering factors through the unique projection to the category of ultrametrics (see Carlsson and M\'{e}moli \cite{cm-2013}), we now see that, in fact, the functor $\sle$ is another description of this projection. In other words, given a weight $(X, u)$, the weight $\sle_u$ is the maximal ultrametric under $u$. Moreover, we can compose with the Rips sieving functor $\rips\colon \wghts \to \sieves$ and recover the single-linkage sieve:
\begin{displaymath}
	\begin{tikzpicture}
		\node (w) at (0,0) {$\wghts$};
		\node (s) at (3, -2) {$\sieves$};
		\node (u) at (0, -2) {$\ults$};
		
		\draw[->, above right] (w) to node {$\slc$} (s);
		\draw[->, left] (w) to node {$\sle$} (u);
		\draw[->, below] (u) to node {$\rips$} (s);	
	\end{tikzpicture}
\end{displaymath}

The construction described in this section gives us a characterization of the maximal ultrametric $\sle_u$ on $X$ underneath a given weight $u$:
\begin{displaymath}
	\sle_u (x_1 ,x_2 ) = \max_{\substack{A \subset X\\x_1 \in A\\x_2 \in A^c}} \min_{\substack{y_1 \in A\\y_2 \in A^c}} u(y_1,y_2),
\end{displaymath}
where $A^c = X \setminus A$. In other words, we look at splits $A, A^c$ of $X$ which separate $x_1$ and $x_2$ and take one where the two sets are as far apart as possible. 

Finally, the clustering domain $\ults$ is characterized (by \Cref{lem:two_points} and \Cref{prop:sle_ult}) as the smallest clustering domain containing all two-point spaces, allowing one to think of hierarchical single linkage clustering as {\em merely the coarsest} among all stationary sieving methods which agree with the Rips sieve on the two point spaces.

\subsection{Examples of Clustering Domains}\label{section:examples of clustering domains}

By enriching the setting of functorial clustering to allow non-partition based methods, we discover that there are many viable clustering domains (and hence stationary sieving functors via \Cref{thm:correspondence}) with various potentially advantageous properties. Here we present a number of examples that help to illuminate the ideas presented in the theoretical results above and which clarify what is, or is not, a clustering domain.  Once we have identified some clustering domains, additional ones can be constructed by taking 
intersections (by intersection, we mean the full subcategory of $\wghts$ with objects given by taking intersection of the objects).  We can also take categories that have the relevant pull-back property under non-expansive maps, but fail to be $\sup$ closed, and ``$\sup$ close'' them to obtain a clustering domain.  Even if the pull-back property fails, one can restrict the 
class of non-expansive maps to so-called admissible (non-expansive) maps where the pull-back property holds.  This necessarily leads to a less restrictive notion of functoriality.  We discuss all of these ideas below.  However, our first example is more in the way of a 
counterexample.

\begin{example}[Cut metrics and tree metrics]

Consider the following metric $d$ on the set  $X=\{1,2,3,4,5\}$:\\

\begin{minipage}{0.49\textwidth}
\begin{displaymath}
 d=\left(
\begin{array}{ccccc}
  0 & 1 & 1 & 1 & 2\\
  1 & 0 & 2 & 2 & 1\\
  1 & 2 & 0 & 2 & 1\\
  1 & 2 & 2 & 0 & 1\\
  2 & 1 & 1 & 1 & 0
 \end{array}
\right)
\end{displaymath}
\end{minipage}
\begin{minipage}{0.49\textwidth}
\begin{displaymath}
\begin{tikzpicture}
	\tikzstyle{vertex}=[circle, draw, fill=black, inner sep=0pt, minimum size=4pt]
	\node[vertex] (1) at (0,0) [label=left:$1$] {};  
	\node[vertex] (2) at (1.5,1) [label=above:$2$] {};
	\node[vertex] (3) at (1.5,0) [label=below:$3$] {};
	\node[vertex] (4) at (1.5,-1) [label=below:$4$] {};
	\node[vertex] (5) at (3,0) [label=right:$5$] {};
	\path
		(1) edge (2)
		(1) edge (3)
		(1) edge (4)
		(5) edge (2)
		(5) edge (3)
		(5) edge (4)
	 ;   
\end{tikzpicture}
\end{displaymath}
\end{minipage}\\
where $d_{ij}$ is the distance from $i$ to $j$.  This metric is the path metric on the graph pictured above where all the edges have length $1$.  \\

It is well known that $(X,d)$ is not a cut metric (see \Cref{def:cut_metric}). Using the machinery developed in \Cref{section:projections}, we will see how this simple concrete example precludes the existence of a functorial projection to the category of cut metrics. Toward this end, consider now the metrics
\begin{displaymath}
 d_1 =\left(
 \begin{array}{ccccc}
  0 & 1 & 1 & 1 & 2\\[2pt]
  1 & 0 & \tfrac{4}{3} & \tfrac{4}{3} & 1\\[2pt]
  1 & \tfrac{4}{3} & 0 & \tfrac{4}{3} & 1\\[2pt]
  1 & \tfrac{4}{3} & \tfrac{4}{3} & 0 & 1\\[2pt]
  2 & 1 & 1 & 1 & 0
 \end{array}
 \right)
\qquad \text{ and } \qquad
 d_0 =\left(
 \begin{array}{ccccc}
  0 & 1 & 1 & 1 & 0\\
  1 & 0 & 2 & 2 & 1\\
  1 & 2 & 0 & 2 & 1\\
  1 & 2 & 2 & 0 & 1\\
  0 & 1 & 1 & 1 & 0
 \end{array}
\right)
\end{displaymath}
which have cut decompositions
\begin{align*}
 d_1 &= \tfrac{1}{3}\delta_{\{1,2\}}+\tfrac{1}{3}\delta_{\{1,3\}} +\tfrac{1}{3}\delta_{\{1,4\}} +\tfrac{1}{3}\delta_{\{2,5\}}+\tfrac{1}{3}\delta_{\{3,5\}}+\tfrac{1}{3}\delta_{\{4,5\}}\\
 d_0 &= \delta_{\{2\}} +\delta_{\{3\}} +\delta_{\{4\}}.
\end{align*}
Then clearly we have $d_0, d_1 < d$, with $d = \max \{d_0, d_1\}$.

Thus the subcategory $\cuts$ of $\wghts$ is not closed under taking max, and so is not a clustering domain. This leaves the question, however, of whether we could ``sup-close'' the category of cut metrics by throwing in some additional metric spaces and arrive at some proper subcategory of $\mets$. The following theorem and its corollaries show that this is not the case; the reasoning used in the five-point example above is in fact general. Indeed, only finite maximums are needed to produce any metric, and the smaller category $\trees$ of tree metrics is sufficient to generate all of $\mets$. Before stating the theorem, we recall the definition of tree metrics (see Dress~\cite{Dress-tight_extensions} for more details on tree metrics). 

\begin{definition}\label{defn:tree,tree metric} A metric space $(Z,d)$ is said to be a {\em tree} if it satisfies
\begin{enumerate}
	\item for every $x,y \in Z$, there is a unique isometric embedding $\varphi_{xy}\colon [0, d_{xy}] \to Z$ with $\varphi_{xy}(0) = x$ and $\varphi_{xy}(d_{xy}) = y$;
	\item for every arc $z\colon [0,1]\to Z$, $t\mapsto z_t$ and for each $t\in[0,1]$ one has $d_{z_0z_t}+ d_{z_tz_1} = d_{z_0z_1}$. 
\end{enumerate}A metric $d'$ on a finite set $X$ is said to be a {\em tree metric on $X$}, if $(X,d')$ embeds isometrically in some tree $(Z,d)$. The category of tree metrics $\trees$ is then the full subcategory of $\mets$ with objects the collection of tree metrics. As above, we will denote the tree metrics on a fixed finite set $X$ by $\trees_X$.
\end{definition}

\begin{theorem} Let $X$ be a non-empty finite set. Then the $\max$-closure of the set $\trees_{X}$ of tree metrics on $X$ is the set $\mets_{X}$ of metrics on $X$.
\end{theorem}
\begin{proof} Since $\mets_{X}$ is $\max$-closed, it suffices to show that every $d\in\mets_{X}$ and every pair $xy\in\binom{X}{2}$ admit a metric $d'\in\trees_X$ such that $d'\leq d$ and $d'_{xy}=d_{xy}$. 

Let $\rho$ denote Kuratowski's isometric embedding of $(X,d)$ in $\ell^\infty(X)$ \cite{kuratowski}, where for all $z\in X$ we have $\rho\colon z\mapsto\rho_z$ and $\rho_z(w)=d(z,w)$. Let $e_x\colon \ell^\infty(X)\to\RR$ denote the evaluation map $e_x(f)=f(x)$ and set 
\begin{displaymath}
	d'_{zw}:=\left|\rho_z(x)-\rho_w(x)\right|=\left|(e_x\circ\rho)(z)-(e_x\circ\rho)(w)\right|\,.
\end{displaymath}
Thus, $d'$ is a pull-back of the standard metric on the real line, hence a tree (semi-) metric. 
Since $e_x$ is non-expansive and $\rho$ is an isometry we also have $d'\leq d$. 
Finally, at the same time we have 
\begin{displaymath}
	d'_{xy}=|\rho_x(x)-\rho_x(y)|=|0-d_{xy}|=d_{xy}\,,
\end{displaymath}
which finishes the proof.
\end{proof}
An immediate corollary is the non-existence of clustering projections~ ---~ and hence, equivalently, of stationary sieving methods~ ---~ whose image coincides (for all underlying sets $X$) with the set of tree metrics:
\begin{corollary}\label{no clustering onto tree metrics} If $\catname{D}$ is a clustering domain satisfying $\catname{D}_X\supseteq\trees_{X}$ for all $X$, then $\catname{D}_X\supseteq\mets_{X}$ for all $X$.\hfill\qedhere
\end{corollary}
Thus, from the point of view of functorial clustering (if based on the metric category), the class of tree metrics~ ---~ the first most accepted (and used) generalization of the class of dendrograms~ ---~ is not suitable for characterizing ``clustered'' objects. Moreover, there exists no ``middle ground'' class of metrics lying between tree metrics and general metrics which could be used for this purpose. In particular, since we have the containments $\trees_{X}\subseteq\cuts_{X}\subseteq\mets_{X}$, there is no hope of obtaining a sieving method characterized by the class of cut metrics:
\begin{corollary}\label{no clustering onto cut metrics} If $\catname{D}$ is a clustering domain satisfying $\catname{D}_X\supseteq\cuts_{X}$ for all $X$, then we have $\catname{D}_X\supseteq\mets_{X}$ for all $X$. \hfill\qedhere
\end{corollary}

\end{example}

We now move on to examples of subcategories which are in fact clustering domains. In \Cref{sec:single_linkage}, we developed an explicit description of a projection functor for single-linkage. Although that projection highlights the naturality of the projection, there is also a computationally efficient algorithm for producing the maximal ultrametric under a given weight: the SLINK algorithm of Sibson~\cite{sibson-73}.  On the other hand, the minimum spanning tree approach to single-linkage clustering in \cite{Gower-Ross:mst} provides an efficient way of producing clusters ({\em i.e.}, the $\slc$ sieving functor). Additionally, the resulting sieves are in this case actually dendrograms, where the sieve itself carries important geometric features (see section \Cref{antipodes} below). This leads us to investigate examples of clustering domains for which:
\begin{enumerate}
	\item[(i)] there is an efficiently computable projection,
	\item[(ii)] the associated sieving functor (after composing the projection with the Rips functor) is computationally tractable,
	\item[(iii)] the associated subcategory of $\sieves$ has objects with a suitable geometric characterization.
\end{enumerate}
In this section, we will identify several examples, focusing on the first two items, while the goal of \Cref{antipodes} is to explore the possibility of extending the third. 

\begin{example}[Metric spaces]
The category $\mets$ itself can be considered a clustering domain in $\wghts$. The canonical projection $\prj_{\mets}$ in this case is obtained by replacing the weights by the distance given by the path metric, {\em i.e.} the length of the shortest path between the points.  This projection is then a well-studied algorithm with several standard computationally efficient approaches. Perhaps most importantly, this example of a clustering domain gives us a way to produce metrics by intersecting with other clustering domains. Care must be taken, however, as in general the projection to an intersection is not as simple as applying one projection and then the other. The associated sieving functor is not known to have an efficient implementation, as it requires computing the maximal simplices of the Rips complex (or equivalently, the maximal cliques of the Rips graph) for arbitrary metric spaces. A geometric characterization of metric sieves is also an open question. 

\end{example}

\begin{example}[Inframetrics]
A finite weight space $(X,u)$ will be called a $\rho$-inframetric space if $u_{xz}\leq\rho\max\{u_{xy},u_{yz}\}$ for every $x,y,z$ in $X$ and $\rho \geq 1$.

Note that when $\rho=1$, we recover the ultrametric condition, and that all metric spaces are $2$-inframetric spaces (but some $2$-inframetrics are not metrics).  It is easy to see that the category of (finite) $\rho$-inframetric spaces is $\sup$ closed and forms a clustering domain.

For $1<\rho<2$, we can intersect the $\rho$-inframetric spaces with $\mets$ to obtain a clustering domain which contains inframetrics inside the category of all metric spaces.  This would yield a projection functor whose associated sieving functor gives a functorial (overlapping) clustering method refining single-linkage. The authors are not aware of tractable algorithms for computing the general $\rho$-inframetric projections (for $\rho > 1$) or the associated sieving functors. Moreover, the geometry of $\rho$-inframetric sieves (those sieves in the image of the associated sieving functor) does not appear to have a known characterization in the non-ultrametric case, when $\rho > 1$. 

In a related vein, we could also work with the $\rho$-relaxed triangle inequality, where $u_{xz}\leq\rho(u_{xy}+u_{yz})$. The $\rho$-inframetric inequality implies the $\rho$-relaxed triangle inequality which in turn implies the $2\rho$-inframetric inequality.  The $\rho$-relaxed inequality also leads to a valid clustering domain.   

\end{example}

\begin{example}[$q$-Metric Spaces] Another example of a clustering domain is the subcategory $\mets_q$ of finite $q$-metric spaces of Segarra {\em et al.} defined in~\cite{Segarra_et_al-q_metric_projections}. For $1\leq q\leq\infty$, a $q$-metric space is a weight space $(X, u)$ where $u$ satisfies
\begin{gather*}
	(u_{xz})^q \leq (u_{xy})^q +(u_{yz})^q \text{ for }q<\infty\\
	\text{or}\\
	u_{xz} \leq \max\{u_{xy},u_{yz}\} \text{ for }q=\infty.
\end{gather*}

When $q=1$ we get ordinary metric spaces and when $q=\infty$ we get ultrametrics.  Note that when $q=2$, we have all triangles being acute.  Recall for ultrametrics all triangles are isosceles with the longest side repeating.  Again it is easy to see that $q$-metric spaces form a clustering domain, and a primary objective of \cite{Segarra_et_al-q_metric_projections} is to investigate the projection for $q < \infty$ defined by $(\prj_{\mets_q}u)_{xy} = \min_{p\in\mathbb{P}(x,y)} \Vert p \Vert_q$,
where $u \in \wghts$, $\mathbb{P}(x,y)$ is the set of edge-paths from $x$ to $y$ in the weighted distance graph, considered as vectors in $\RR^k$ (with $k$ the length of the path) and $\Vert \cdot \Vert_q$ the usual $q$-norm. An efficient algorithm for the sieving functor and a geometric understanding of $q$-metric sieves are still needed. 

\end{example}

\begin{example}[Discretizations] Let $\catname{Int}$ be the category of finite weight spaces with integer distances.  This is obviously $\sup$ closed in the sense that given any finite set $X$ with a weight $u$ there exists a unique maximal integer weight under $u$.  When restricting to metrics ({\em i.e.}, intersecting $\catname{Int} \cap \mets$), we must be careful to first take the integer floor of the distances given by $u$ and then take the resulting path metric on $X$ with those integer weights.  The latter step is necessary since the floor of $u$ may not be a metric. Also note that reversing this procedure is not necessarily the same. 

Notice that this example can be generalized in a straightforward way using any given sequence of non-negative real numbers, providing a mechanism for discretizing in a way appropriate for a given application that does not destroy the underlying theoretical framework for clustering. Generically, an efficient algorithm for the sieving functor is bounded in complexity by an efficient algorithm for Rips sieving, since combinatorially these will produce the same sieve.  

\end{example}

\begin{example}[Quotient spaces] Let $X$ be a fixed finite set and suppose $\sim$ is an equivalence relation on $X$ with quotient $Z$. Let $\pi\colon X \to Z$ denote the quotient map. Then $\pi^{\ast}\colon \mets_{Z} \to \mets_{X}$ embeds $\mets_{Z}$ in $\mets_{X}$ as the set of all metrics satisfying $w_{xy} = 0$ whenever $x \sim y$ in $X$. Observe that $\pi^\ast\mets_{Z}$ is a closed and max-closed subset of $\mets_{X}$, hence also sup-closed, by \Cref{prop:closed_max_closed}. Define $\fname{F}_X\colon \mets_{X}  \to \mets_{X}$ by $\fname{F}_X(w) = \sup \{u \in\pi^\ast\mets_{Z} \mid u \leq w\}$. Then $\fname{F}_X^2 = \fname{F}_X$, $\fname{F}_X(\mets_{X}) = \pi^\ast\mets_{Z}$, $\fname{F}_Xw \leq w$, and $w_1 \leq w_2$ implies $\fname{F}_Xw_1 \leq \fname{F}_Xw_2$ for all $w, w_1, w_2 \in \mets_{X}$. Any map $\fname{H}\colon\mets_{X} \to \mets_{X}$ satisfying these conditions must coincide with $\fname{F}_X$. 

For any clustering domain $\catname{D}$ contained in $\mets$, we have the diagram:
\begin{displaymath}
	\begin{tikzpicture}
		\node (ul) at (0,3) {$\mets_{Z}$};
		\node (ur) at (3,3) {$\mets_{X}$};
		\node (ll) at (0,0) {$\mets_{Z}$};
		\node (lr) at (3,0) {$\mets_{X}$};
		
		\draw[->, above] ([yshift=.5em]ur.west) to node {${\pi^{\ast}}^{-1} \circ \fname{F}_X$} ([yshift=.5em]ul.east);
		\draw[->, below] ([yshift = -.5em]ul.east) to node {${\pi^{\ast}}$} ([yshift=-.5em]ur.west);
		\draw[->, left] (ul) to node {$(\prj_{\catname{D}})_Z$} (ll);
		\draw[->, right] (ur) to node {$(\prj_{\catname{D}})_X$} (lr);
		\draw[->, below] (lr) to node {${\pi^{\ast}}^{-1} \circ \fname{F}_X$} (ll);
		
		\draw[->, red, left, rounded corners=7mm] (.3, 2.4) -- (2.8,2.4) -- node {$\fname{G}_X$} (2.8, .3) -- (.3,.3);
	\end{tikzpicture}
\end{displaymath}

We claim that $\fname{G}_X = {\pi^{\ast}}^{-1} \circ \fname{F}_X \circ (\prj_{\catname{D}})_X \circ {\pi^{\ast}}$ coincides with the map $(\prj_{\catname{D}})_Z$. Indeed, it satisfies all $3$ requirements for characterizing $(\prj_{\catname{D}})_Z$. Thus $\fname{F}_X$ is a natural ``quotient operation'' on metric spaces that commutes with $\prj_{\catname{D}}$.

Recall that, for any metric $w\in\mets_X$ the quotient metric $\bar{w}$ on the quotient $Z$ is defined as follows (see, for example, Section~$3.3$ in~\cite{bbi}). First, for $A,B\in Z$ one considers `paths' from $A$ to $B$:
\begin{displaymath}
	\mathbb{P}(A,B) := \left\{ (x_o, y_1, x_1, \ldots, y_{n-1}, x_{n-1}, y_n) ~\left|~
	\begin{aligned}
	&x_0 \in A, y_n \in B, n \in \mathbb{N},\\
	&x_i \sim y_i \text{ for } i = 1, \ldots, n
	\end{aligned}
	\right.
	\right\}
\end{displaymath}
Next, for each $p \in \mathbb{P}(A,B)$ as above one defines its length as $\ell(p) := \sum_{i = 1}^n w(x_{i-1}, y_i)$. Finally:
\begin{displaymath}
	\bar{w}_{AB} := \inf_{p \in \mathbb{P}(A,B)} \ell(p). 
\end{displaymath}
Since $\pi^{\ast}(\bar{w}) \leq w$ for all $w\in\mets_X$, $\fname{F}_X(w)$ coincides with $\pi^{\ast}(\bar{w})$.  This example is a bit different from those above, but the interesting point here is that quotient metrics respect {\em any} stationary functorial sieving map, providing yet another reason to adopt this particular generalization of hierarchical clustering.

\end{example}

\begin{example}[Any subcategory of $\wghts$] We remind the reader that any subcategory of \textbf{W} with the pull-back property can be $\sup$ closed to give a clustering domain.  The difficult issue is to understand exactly what that closure is and how to construct the projection to it in a practical way. With the ability to take intersections of clustering domains, it is clear that there are many such subcategories, providing different clustering methods. Much more work is needed to characterize additional useful clustering domains. 

\end{example}

\section{Injective envelopes and the role of antipodes in clustering}\label{antipodes}

The significance of \Cref{no clustering onto tree metrics,no clustering onto cut metrics} is in their prohibiting a functorial compromise between the classical notion of hierarchical clustering and the geometric clustering maps based on split decompositions championed by the Dress school. However, the rather complicated relation between split-based clustering methods and the geometry of injective envelopes inspires hope that a sieving functor having some of the qualities of split-based maps could be derived from the geometry of the injective envelopes of spaces chosen to lie in an appropriately defined clustering domain. Setting aside all technical details for a separate account, we focus in this section on introducing and reviewing relevant properties of the clustering domain of {\em $\antip{}$-spaces}, motivated by a study of rooted injective envelopes (see below), as well as some additional clustering sub-domains which, we believe, will prove useful as a means of constructing new classes of hierarchical classifiers lying between the category of dendrograms provided by $\slc$ and the category of unrestricted sieves provided by $\rips$.

\subsection{Injective envelopes and clustering}
Recall\footnote{For a modern, self-cointained account of injective envelopes see~\cite{Lang-injective_envelopes_and_groups}.} that a metric space $X$ is said to be {\em injective}, if, for any isometry $i\colon A\to B$, any non-expansive map $f\colon A\to X$ extends to a non-expansive map $F\colon B\to X$ (in the sense that $F\circ i=f$). As any metric space $X$ embeds isometrically in an injective one (consider Kuratowski's Banach embedding~\cite{kuratowski}), one asks whether an {\em injective envelope} $\epsilon X$ exists, that is: an isometry $e\colon X\to\epsilon X$ into an injective space $\epsilon X$, through which any embedding of $X$ to an injective space must factor. The construction of $\epsilon X$, discovered independently by Isbell~\cite{isbell} and Dress~\cite{Dress-trees_and_tight_spans}, is explicit and easy to describe: first setting
\begin{displaymath}
	P(X,d):=\set{f\colon X\to\RRplus}{
		\forall {x,y\in X}\;
		f(x)+f(y)\geq d_{xy}
	}
\end{displaymath}
one then lets $\epsilon X$ be the subset of those $f\in P(X,d)$ that are pointwise minimal, inheriting its metric from the $\sup$-metric. The mapping $e\colon X\to\epsilon X$ defined by $e(x)(y)=d_{xy}$ is the required embedding.

The injective envelope $e\colon X\to\epsilon X$ provides a canonical way to minimally ``fill-in'' $X$ so that as many of its points as possible appear as endpoints of the complete and hyper-convex~\cite{aronszajn_panitchpakdi-extension} metric space $\epsilon X$, prompting the idea that cluster hierarchies $X$ might be derived from connectivity properties of $\epsilon X$. For example, computing the cut-point hierarchy~\cite{Ward-axioms_for_cutpoints} of $\epsilon X$~\cite{DHM-cut_points_in_envelope} induces a natural nested tree-like structure on $X$ in the sense of~\cite{Dicks_Dunwoody-groups_acting_on_graphs}. The study of the relation between envelopes and split decompositions by the Dress school~\cite{bandelt-dress,dress_buneman,Dress_et_al-comparison_median_tight_span,dmsw-2013} following Bunemann's work~\cite{Buneman-recovery_of_trees} pushes this idea even further, but seems to lose track of the projective/hierarchical aspects of the clustering problem. 

\subsection{Antipodes}
Our present outlook on the usefulness of injective envelopes is motivated by two well known facts. First, any dendrogram $(X,\theta)$~ ---~ when realized as a metric tree with leaf set $X$, in which the length of each edge is equal to half the absolute value of the difference in heights between its endpoints~ ---~ is naturally isometric to the injective envelope of $(X,u)$, where $u\in\ults_{X}$ is the ultra-metric corresponding to the given dendrogram, $(X,\theta)=\rips(X,u)$. Second, a metric tree is the geometric realization of a dendrogram if and only if it contains a {\em root vertex}: a point at equal distances from all the points of $X$ (also known as the {\em leaves} of the dendrogram). We are led to investigate (compact metric) spaces with the property that their injective envelope has a unique root in the same sense. It turns out that these are precisely the (compact, metric) spaces where every point has an antipode, or {\em $\antip{}$-spaces}:
\begin{definition}[Antipode, $\antip{}$-space\footnote{The notion ``Antipodal Space'' has already been put to extensive good use in \cite{HKM-tight_span_of_antipodal1,HKM-tight_span_of_antipodal2}}] A metric space $(X,d)$ is defined to be an \defn{$\antip{}$-space} if every $x\in X$ has an {\em antipode}, that is: there exists $y\in X$ satisfying $d_{xy}=\diam{X,d}$. Denote the set of $\antip{}$-spaces with underlying set $X$ by $\ants{}_X$.
\end{definition}

\begin{remark}\label{rem:ultra vs A-space}
For finite $X$, the collection $\ants{}_X$ is a closed subspace of $\wghts_{X}$, and is clearly $\max$-closed, hence also $\sup$-closed.
Note that $(X,d)\in\ults$ if and only if every subspace of $(X,d)$ is an $\antip{}$-space, or equivalently, if every triple in $(X,d)$ is an $\antip{}$-space. In particular, an ultrametric space is globally an $\antip{}$-space.
\end{remark}

Overall, we conclude that the class of $\antip{}$-spaces forms a clustering domain fibering over the the category of finite sets and {\em surjective} set maps. The surjectivity restriction is necessary, because only for surjective set maps is it true that the pull-back of an $\antip{}$-metric is again an $\antip{}$-metric. The surjectivity restriction is also more than a mere convenience to accommodate the category-theoretical lingo. In fact, it results in relaxing the consistency requirements on the corresponding sieving functor from $\mets$ to $\sieves$, implying that the resulting clustering method is consistent with respect to quotients, but not necessarily consistent with respect to sub-sampling (which corresponds to pull-backs via injective maps). The property of $X$ being an $\antip{}$-space is, in fact, encoded in the geometry of its injective envelope, which we summarize in the following theorem. The proof is straightforward from the definition of $\antip{}$-spaces and the geometry of injective envelopes.
\begin{theorem} Given a compact metric space $(X,d)$, let $\hat X$ denote the extension $X\cup\{\infty\}$, $\infty\notin X$, adding $d_{\infty x}:=\tfrac{1}{2}\diam{X,d}$ for all $x\in X$. The following are equivalent:
\begin{enumerate}[leftmargin=35pt]
	\item $(X,d)$ is an $\antip{}$-space;
	\item $\epsilon\hat X$ is isometric to $\epsilon X$ through an isometry fixing $X$ pointwise;
	\item $\epsilon X$ contains a point~ ---~ denoted $\infty_X$~ ---~ at equal distances to all points of $X$;
	\item $\epsilon X$ contains a point at a distance $\tfrac{1}{2}\diam{X,d}$ to every $x\in X$.
\end{enumerate}
The point $\infty_X$ will be referred to as the {\em root} of $(X,d)$ in $\epsilon X$.\hfill\qedhere
\end{theorem}
The example of a dendrogram viewed as the injective envelope of an ultra-metric space $(X,d)$ provides a hint at the class of subspaces one might want to regard as clusters for the general $\antip{}$-space. When $\epsilon X$ is a dendrogram, each $f\in\epsilon X$ is no more than a specification of distances to the leaves, with the leaves closest to $f$~ ---~ and hence of minimal value under $f$~ ---~ forming the associated `descendant' cluster. See \Cref{fig:4pt_antipodal_spaces_comparison} for a comparison in four-point spaces.
\begin{figure}
	\begin{center}
		\includegraphics[width=\textwidth]{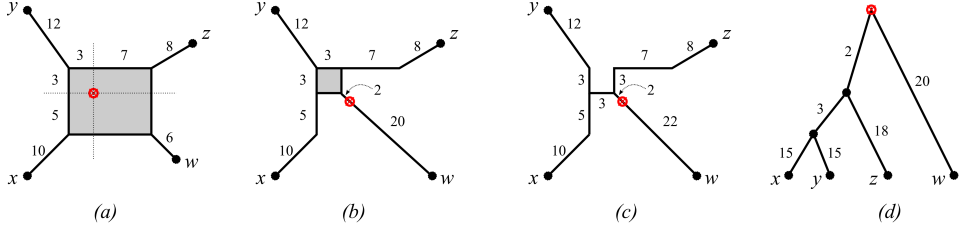}
		\caption{Envelopes of the two generic types of $4$-point $\antip{}$-spaces (a,b), seen as `fattened' versions of a dendrogram (c), also drawn in a more standard representation in (d). The root of each envelope is marked with a red $\otimes$. For a more meaningful comparison, the metric in (a-c) is chosen to be the path metric induced from embedding in the $\ell_1$ plane.\label{fig:4pt_antipodal_spaces_comparison}}%
	\end{center}
\end{figure}

Denote the minimum value of $f$ by $\height{f}$, and the set of points where $\height{f}$ is achieved by $\minset{f}$. For a general $\antip{}$-space we have:
\begin{proposition}\label{prop:minsets in A-spaces} Suppose $(X,d)$ is an $\antip{}$-space, then every $f\in\epsilon X$ satisfies
\begin{displaymath}
	\dist{\infty_X}{f}=\tfrac{1}{2}\diam{X,d}-\height{f}
\end{displaymath}
Moreover, every $f\in\injenv{X,d}$ satisfies
\begin{displaymath}
	\max_{z\in X}f(z)+\height{f}=D
\end{displaymath}
In particular, $f(x)+f(y)=D$ whenever $x\in\minset{f}$ and $y$ is an antipode of $x$.\hfill\qedhere
\end{proposition}

Thus, $\antip{}$-spaces form a clustering domain containing the ultrametric clustering domain and generalizing some pertinent geometric properties of ultrametrics. Moreover:
\begin{remark}\label{rem:projection to A-spaces} Given a metric space $(X,d)$, its projection to the category of $\antip{}$-spaces may be computed recursively as follows. Set $d_0=d$ and for any $t\geq 0$ define the set $E_t=\{xy\,|\,d_t(x,y)=\diam{X,d_t}\}$. If $E_t$ is an edge cover then stop and return $d_t$, as $(X,d_t)$ is an $\antip{}$-space if and only if $E_t$ is an edge cover of $X$. Else, for each $xy\in E_t$ set $d_{t+1}(x,y)$ to equal the second-largest distance in $(X,d_t)$; for $xy\notin E_t$ set $d_{t+1}(x,y)=d_t(x,y)$. 
\end{remark}

\subsection{Clustering domains of $\antip{}$-spaces}\label{sec:a_spaces}

\subsubsection{$4$-point conditions, and more}
Additional classes of $\antip{}$-spaces exist, forming clustering domains which contain the domain of ultra-metrics.
\begin{definition}\label{defn:Am space} Let $m\geq 3$ be an integer. We say that a metric space $(X,d)$ is an \defn{$\antip{m}$-space}, if every subset of cardinality $m$ is an $\antip{}$-space.
\end{definition}
The following lemma is easy to prove:
\begin{lemma}\label{Am space is Am+1} Let $m\geq 3$ be an integer. If $(X,d)$ is an $\antip{m}$-space, then it also an $\antip{n}$-space for any $n\geq m$. In particular, a finite $\antip{m}$-space of cardinality at least $m$ is an $\antip{}$-space.\hfill\qedhere
\end{lemma}
It is also straightforward to observe that pull-backs of $\antip{n}$-spaces under injective maps remain $\antip{n}$-spaces, as well as that the set of $\antip{n}$-spaces with a fixed finite base space $X$ is a closed and max-closed subset of $\wghts_{X}$. Thus we have the following corollary for $\antip{n}$-spaces. 
\begin{corollary}\label{Am spaces form a clustering domain} For every integer $m\geq 3$, the class of finite metric spaces that are $\antip{m}$-spaces forms a clustering domain over the category of sets and {\em injective} maps. This clustering domain will be denoted by $\ants{m}$.\hfill\qedhere
\end{corollary}
The resulting tower of clustering domains over the category of finite sets and {\em injective} set maps,
\begin{displaymath}
	\ants{3}\subsetneq\ants{4}\subsetneq\cdots\subsetneq\cdots\ants{m}\cdots\,,
\end{displaymath}
is bounded above by the {\em class} $\ants{}$ of compact $\antip{}$-spaces, though we must be careful not to view $\ants{}$ as a clustering domain in this context, since the morphism structures of $\ants{m}$ and $\ants{}$ cannot be reconciled. Similarly, while $\ants{3}=\ults$ as a class of objects (see \Cref{rem:ultra vs A-space}), the inclusion map from $\ults$ to $\ants{4}$ is not a functor, as $\ults$ admits all non-expansive set maps---not just the injective ones. Because the poset structure of the fibers are consistent (see \Cref{rem:fibered}), however, we are still able to consider the factorization of the projections $\prj_{\ants{m}}$ through $\prj_{\ants{}}$.

Thus, the sieving functors from $\ants{m}$ to $\sieves$ should be viewed as functors which are only consistent with respect to sub-sampling. We have yet to obtain a projection algorithm to, say, $\ants{4}$-spaces, or a geometric characterization of their injective envelopes.

\subsection{Discussion: Injective envelopes and practical clustering}

The view of a sieving functor as a projection in the weight/metric category followed by Rips clustering is attractive from the point of view of it providing a relatively simple and geometrically motivated means for constructing such functors {\em abstractly}, but from a computational standpoint it is a bit na\"ive. Indeed, the computational complexity of Rips sieving is, essentially, prohibitive for large data sets. One therefore is after clustering domains whose geometric properties enable efficient algorithms for computing the clustering directly.

In this context, the domain of ultra-metrics provides an extreme example, where computing the sieve associated with a given metric space is not just efficient, but can be efficiently distributed/decentralized.

On the other extreme, the domain of $\antip{}$-spaces seems to offer little improvement over Rips sieving of a general metric: every finite metric space $(X,d)$ extends to an $\antip{}$-space by adding just one more point (for instance, adding a point that has distance $\diam{X,d}$ from every point in $X$ is sufficient); in other words, though geometrically well-motivated, the requirement that $(X,d)$ be an $\antip{}$-space is not sufficiently restrictive. The algorithm in \Cref{rem:projection to A-spaces} suggests that unless the projection $(X,d_A)$ of $(X,d)$ to the domain of $\antip{}$-spaces yields a clear and roughly even {\em partitioning} of $X$ into clusters at some degree of resolution, Rips-clustering of $(X,d_A)$ will still require a search through a significant portion of the $2^n$ possible clusters.

Clearly, this `malfunction' of sieving through $\antip{}$-spaces is due to a lack of some kind of hereditary/hierarchical structure in spaces in this category: no particular collection of subspaces of an $\antip{}$-space is forced to inherit the $\antip{}$-space condition. This motivates the introduction of the notion of $\antip{m}$-spaces, best viewed as a relaxation of the hereditary class of ultra-metric spaces~ ---~ $\antip{4}$-spaces being of particular interest, in view of the key role of $4$-point conditions in the metric clustering literature~\cite{Buneman-recovery_of_trees,Dress-trees_and_tight_spans}.

Judging from the ultra-metric case, it seems plausible that, for a clustering domain $\catname{D}$, the existence of an efficient sieving algorithm requires a combination of (1) geometric constraints that force some notion of ``thinness'' on $\epsilon(X,d)$ whenever $(X,d)\in\catname{D}$; (2) all clusters of $(X,d)\in\catname{D}$ lying in $\catname{D}$, similarly to the notion of {\em excisiveness} introduced in \cite{cm-2013}; and (3) proper restrictions on the base morphisms. The extent to which this vague conjecture holds true for $\ants{m}$, $m\geq 4$, is the subject of ongoing work. 

\section{Acknowledgements} The authors gratefully acknowledge the support of Air Force Office of Science Research under the LRIR 15RYCOR153, MURI FA9550-10-1-0567 and FA9550-11-10223 grants, respectively. Additionally, the authors would like to express their sincere appreciation to the reviewers, whose comments and suggestions have produced more clarity and precision in the exposition.

\bibliographystyle{plain}
\bibliography{./functorial_clustering}

\end{document}